\documentclass{article}

\usepackage[shrink=28]{microtype}
\usepackage{graphicx}
\usepackage{caption}
\usepackage{subcaption}
\usepackage{booktabs} 
\usepackage{mathrsfs}

\usepackage{nicefrac}
\usepackage{xcolor}
\usepackage{colortbl}
\usepackage{amsmath,amssymb,amsthm}

\usepackage{listings}
\definecolor{codegreen}{rgb}{0,0.6,0}
\definecolor{codegray}{rgb}{0.5,0.5,0.5}
\definecolor{codepurple}{rgb}{0.58,0,0.82}
\definecolor{backcolour}{rgb}{0.95,0.95,0.92}

\lstdefinestyle{mystyle}{
    backgroundcolor=\color{backcolour},   
    commentstyle=\color{codegreen},
    keywordstyle=\color{magenta},
    numberstyle=\tiny\color{codegray},
    stringstyle=\color{codepurple},
    basicstyle=\ttfamily\footnotesize,
    breakatwhitespace=false,         
    breaklines=true,                 
    captionpos=b,                    
    keepspaces=true,                 
    numbers=left,                    
    numbersep=5pt,                  
    showspaces=false,                
    showstringspaces=false,
    showtabs=false,                  
    tabsize=2
}
\usepackage{thm-restate}

\theoremstyle{definition}

\newtheorem*{proposition*}{Proposition}

\usepackage{url}

\usepackage{amsmath,amsfonts,amssymb,euscript,stmaryrd,graphicx,enumitem,yhmath,mathrsfs,xspace, empheq}
\usepackage{xcolor}

\usepackage{enumitem}
\setlist[enumerate]{leftmargin=0.5cm,topsep=0pt,itemsep=-2pt}
\setlist[itemize]{leftmargin=0.5cm,topsep=0pt,itemsep=-2pt}

\usepackage{mathrsfs}

\DeclareMathOperator*{\argmin}{arg\,min}

\usepackage{mathtools}
\usepackage{hyperref}
\hypersetup{
    colorlinks=true,
    linkcolor=blue,
    citecolor=cyan,
}

\usepackage[capitalize]{cleveref}
\Crefname{assumption}{Assumption}{Assumptions}
\Crefname{equation}{Equation}{Equations}
\creflabelformat{equation}{(#2#1#3)}

\usepackage[accepted]{icml2023}

\newcommand{\vertiii}[1]{{\left\vert\kern-0.25ex\left\vert\kern-0.25ex\left\vert #1 
    \right\vert\kern-0.25ex\right\vert\kern-0.25ex\right\vert}}

\newcommand{\byol}{\texttt{BYOL}\xspace}
\newcommand{\simsiam}{\texttt{SimSiam}\xspace}

\begin{document}

\twocolumn[
\icmltitle{The Edge of Orthogonality: A Simple View of What Makes BYOL Tick}
\icmlsetsymbol{equal}{*}

\begin{icmlauthorlist}
\icmlauthor{Pierre H. Richemond}{dm}
\icmlauthor{Allison Tam}{dm}
\icmlauthor{Yunhao Tang}{dm}
\icmlauthor{Florian Strub}{dm}
\icmlauthor{Bilal Piot}{dm}
\icmlauthor{Felix Hill}{dm}
\end{icmlauthorlist}

\icmlaffiliation{dm}{DeepMind}

\icmlcorrespondingauthor{Pierre Harvey Richemond}{richemond@deepmind.com}

\icmlkeywords{Machine Learning, ICML}

\vskip 0.3in
]



\printAffiliationsAndNotice{} 

\begin{abstract}
Self-predictive unsupervised learning methods such as \byol~\citep{Grill2020} or \simsiam~\citep{SimSIAM} have shown impressive results, and counter-intuitively, do not collapse to trivial representations. In this work, we aim at exploring the simplest possible mathematical arguments towards explaining the underlying mechanisms behind self-predictive unsupervised learning. We start with the observation that those methods crucially rely on the presence of a predictor network (and stop-gradient). With simple linear algebra, we show that when using a linear predictor, the optimal predictor is close to an orthogonal projection, and propose a general framework based on orthonormalization that enables to interpret and give intuition on why \byol works. In addition, this framework demonstrates the crucial role of the exponential moving average and stop-gradient operator in \byol as an efficient orthonormalization mechanism. We use these insights to propose four new \emph{closed-form predictor} variants of \byol to support our analysis. Our closed-form predictors outperform standard linear trainable predictor \byol at $100$ and $300$ epochs (top-$1$ linear accuracy on ImageNet).
\end{abstract}

\section{Introduction}
\label{sec:intro}
Modern Self-Supervised Learning (SSL) methods~\citep{Grill2020,SimCLR} have all but closed the performance gap with supervised methods in computer vision. Famously, these methods also form the basis for multimodal representation algorithms that power useful downstream applications such as generative modelling at scale~\citep{Radford2021LearningTV, Jia2021ScalingUV, Ramesh2022HierarchicalTI, Alayrac2022FlamingoAV}. Amongst SSL methods broadly two families of algorithms coexist. First, \textbf{contrastive} methods iteratively push the representations of \emph{positive pairs} of examples closer together, and pull \emph{negative pairs} of examples further apart. They optimize a lower bound on the mutual information between representations and data~\citep{CPC}.

Second, \textbf{self-predictive} methods rely on minimizing a least-squares objective between two network outputs with two different augmented views of the same inputs. This may involve two related or \emph{Siamese}~\citep{SiameseNetworks} neural network streams with or without the same architecture~\citep{Grill2020, SimSIAM, Ermolov2021}. For instance, \byol features both an \emph{online} network, composed of an encoder and a predictor and a \emph{target} network, composed of a similar encoder but not sharing weights. The target encoder is updated through an exponential moving average of the online encoder, whereas the online network (encoder and predictor) is updated through gradient descent by minimizing a \byol loss between the stop-gradiented target encoder outputs and online predictor outputs. Both outputs are computed from the same inputs with different augmentations. Self-predictive learning is more memory-efficient than contrastive learning as it does not rely on extremely large batch sizes (due to negative examples) to reach good performance. However, it was initially surprising and opaque why those algorithms learn anything at all rather than collapsing to a single all-attracting constant value. Yet, several works started unveiling the underlying mechanisms behind self-predictive unsupervised learning~\citep{Tian2021UnderstandingSL,bridgingthegap, Eigenspaceview} (see Sec.~\ref{sec:related_work}). Here we propose simple explanations for this phenomenon through linear algebra, and use those to derive new, closed-form predictors for \byol. 

The crucial starting point is to notice that \byol can be seen as matrix trace maximization, and we here refine this argument. Several fundamental algorithms in machine learning such as PCA~\citep{PCAPearson},  spectral clustering~\citep{Cheeger1969ALB, Shi1997NormalizedCA}, or non-negative matrix factorization~\citep{Paatero1994PositiveMF, Lee1999LearningTP}, share this property and are also readily written as least-squares minimization objectives. Crucially they do not collapse \emph{thanks to optimization constraints}~\citep{Kokiopoulou2011TraceOA}. For instance, the variational characterization of PCA implies maximization over orthogonal matrices~\citep{Horn2013MatrixA}. In this work and as core contribution, we argue that a linear predictor version of \byol is no different, but imposes predictor constraints that are both \emph{approximate} and \emph{implicitly enforced by its architecture}. Importantly, we justify these points using standard linear algebra, specifically, the optimal linear predictor is close to an orthogonal projection. This subsumes prior work~\citep{Tian2021UnderstandingSL, bridgingthegap, Eigenspaceview} and enables a more general view informed by optimization on Riemannian manifolds of orthogonal matrices~\citep{Edelman1998TheGO, Absil2007OptimizationAO, Bonnabel2013StochasticGD}. We leverage this connection to propose several new self-predictive self-supervised algorithms with \emph{closed-form predictors}.

\textbf{Contributions.} Our contributions are as follows:
\begin{itemize}
    \item We investigate the role of the linear predictor in \byol and show its role as an approximate \emph{orthogonal projection} as well as a spectral expansion operator for the covariance of latents.
    \item We interpret the \byol update as a Riemannian gradient descent step, where the expensive \emph{retraction} step is made computationally trivial by the predictor and EMA.
    \item We use those principles as inspiration to propose new self-supervised learning methods with four closed-form predictors, summarized Table \ref{summarypredictors}. We evaluate those empirically on ImageNet-$1$k, with a ResNet-$50$ encoder, at $100$ and $300$ epochs. Our closed-form predictors outperform standard linear trainable predictor \byol while only using matrix multiplications.
\end{itemize}

Besides unifying and simplifying our understanding of self-predictive unsupervised learning, these theoretical insights yield methods that can be more data- and compute-efficient, thus application-friendly.

\vspace{-1.0em}
\section{Background and notations}
Given batch size $b\in\mathbb{N}$, input dimension $d\in\mathbb{N}$ and latent dimension $f\in\mathbb{N}$, we consider batches of datapoints $X=(x^i)_{1\leq i\leq b}$, where the batch $X$ can be seen as a real matrix of shape $(b,d)$ and each datapoint $x^i\in\mathbb{R}^d$ as a real vector. The datapoints $x^i=t^i(o^i)$ are transformations of the original raw inputs $o^i\in\mathbb{R}^d$ where the augmentations $t^i\sim\mathcal{T}$ are chosen uniformly from distribution $\mathcal{T}$. Similarly $X^{\prime}=(x^{\prime i}=t^{\prime i}(o^i))_{1\leq i\leq b}$ is augmented from distribution $\mathcal{T}^{\prime}$. These are passed through the combined operation of an encoder and a projector, denoted as $A_\theta$ for the online branch and $A_\xi$ for the target branch, resulting in latents $Z_\theta=A_\theta(X)$ and $Z^{\prime}_\xi=A_\xi(X^{\prime})$. The latents $Z_\theta=(z_{\theta}^i)_{1\leq i\leq b}$ and $Z^\prime_\xi=(z^{\prime i}_\xi)_{1\leq i\leq b}$ are real matrices of shape $(b,f)$ with respective row vectors $z^i_\theta\in\mathbb{R}^f$ and $z^{\prime i}_\xi \in \mathbb{R}^f$. The predictor is denoted by $P_\theta$. The operator norm of a matrix $X$ is noted $\vertiii{X}$, its Frobenius norm $\left\| X \right\|_{F}$. Vector norms $\left\| x \right\|_2$ are Euclidean. The target parameter $\xi$ is updated via exponential moving average with parameter $\tau$ and the online parameter $\theta$ is updated through gradient descent via the  \byol objective.
 
\textbf{BYOL objective}. Given a predictor $P_\theta$ and two distributions of latents $Z_\theta$ and $Z_\xi^\prime$, the \byol objective is the expected sum over the batch of the normalized $L^2$ distance between the target outputs $z_{\xi}^{\prime i}$ and the online predictor outputs $P_\theta(z^i_{\theta})$:
\begin{equation}\label{byolgeneralobjective}
L_{\text{BYOL}}(\theta) = \mathbb{E}\left[\sum_{i=1}^b\left\| \frac{P_\theta(z^i_{\theta})}{\|P_\theta(z^i_{\theta})\|_2} - \texttt{sg}\left(\frac{z_{\xi}^{\prime i}}{\|z_{\xi}^{\prime i}\|_2}\right)  \right\|^{2}_{2} \right]
\end{equation}
where $\texttt{sg}(.)$ is a stop-gradient operator preventing backpropagation through parameters $\xi$. The stop-gradient notation is redundant and may be omitted, as the \byol objective is optimized with respect to $\theta$ only.

\section{A focus on the linear predictor}
Despite the simplicity of the \byol objective (\ref{byolgeneralobjective}),  prior work on \byol theory has involved non-trivial matrix ordinary differential equation analysis under strong assumptions \citep{Tian2021UnderstandingSL}. We here present proofs with simpler linear algebra arguments. To do so, we make the two following simplifying assumptions from now on. First, we assume a linear predictor $P_\theta(Z_\theta) = Z_{\theta} P_\theta$, where $P_\theta \in \mathbb{R}^{f \times f}$ is a real square matrix. Second, we use the unnormalized $L^2$ distance rather than the normalized variant in our theoretical analysis. These two changes only worsen top-$1$ accuracy by a few percentage points as shown in \citet{Grill2020}. Therefore, the \byol objective becomes:
\begin{equation}\label{byolobjective}
    L_{\text{BYOL}}(\theta) = \mathbb{E}\left[\sum_{i=1}^b\left\|  z^i_{\theta} P_\theta  -z _{\xi}^{\prime i}  \right\|_{2}^{2}\right] = \mathbb{E}\left[\left\|  Z_{\theta} P_\theta  - Z _{\xi}^{\prime} \right\|_{F}^{2}\right], 
\end{equation}
by definitions of $Z_\theta$, $Z_{\xi}^{\prime}$ and the Frobenius norm.
Under this form we can also see the \byol objective as a matrix trace:
\begin{align}\label{byoltrace}
    L_{\text{BYOL}}(\theta)
            &= \operatorname{tr} \left[ \mathbb{E} \big( (Z_{\theta} P_\theta -Z _{\xi}^{\prime})^{\top} (Z_{\theta} P_\theta -Z _{\xi}^{\prime}) \big) \right].
\end{align}
We begin with proving a crucial property of the linear optimal projector, namely, that it is an orthogonal projector (Prop.~\ref{prop:one}). We illustrate empirically that the subspace being projected upon grows throughout training, until the optimal predictor becomes the identity matrix at the end of training. This means \byol does not collapse, and its feature space grows and, under conditions (Prop.~\ref{predictoridentity}), retains as much information as possible. Using these insights we propose a new closed-form expression for a batchwise optimal predictor.

\subsection{Optimal linear predictor and orthogonal projection}

As described in~\citet{Grill2020}, the optimal predictor induces the following conditional expectation:
\begin{equation*}
P^{\star}(z_{\theta}) \triangleq \underset{P}{\arg \min }\quad  \mathbb{E}\left[\left\|P\left(z_{\theta}\right)-z_{\xi}^{\prime}\right\|^{2}_2\right] = \mathbb{E}\left[z_{\xi}^{\prime} \mid z_{\theta}\right].
\end{equation*}
When the predictor is both optimal and \emph{linear}, it performs exactly a linear regression, on a batch-by-batch basis, and so is given by:
    \begin{equation}\label{closed_form_lrp_predictor}
P^{\star}=\argmin_P \quad \left\|Z_{\theta} P - Z_{\xi}^{\prime}\right\|_{F}^{2} = \left(Z_{\theta}^{\top} Z_{\theta}\right)^{-1} Z_{\theta}^{\top}  Z_{\xi}^{\prime}.
\end{equation}
We refer to this as $\mathbf{LRP}$ (linear regression predictor). We can use this setting to derive an important property of the batchwise-optimal linear predictor, as per proposition below:
\vspace{-0.5em}
\begin{restatable}{prop}{propone}\label{prop:one}
The linear optimal predictor, batchwise, is an orthogonal projection.
\end{restatable}
\emph{Proof:} Inserting the right side of equation \ref{closed_form_lrp_predictor} into the \byol loss, reveals the term $Z_{\theta} \left(Z_{\theta}^{\top} Z_{\theta}\right)^{-1} Z_{\theta}^{\top}$. That term is the orthogonal projection on the subspace spanned by the batch $Z_\theta$, we denote it by 
$\operatorname{Proj}_{(Z_\theta)} = Z_{\theta} \left(Z_{\theta}^{\top} Z_{\theta}\right)^{-1} Z_{\theta}^{\top}$. The identity minus that term is also an orthogonal projection, on the orthogonal complement of that space, denoted by 
$\operatorname{Proj}_{(Z_\theta)^{\bot}}$. With such notations the \byol loss for an optimal linear predictor $P^{*}$ becomes:
\begin{equation}\label{projectororthogonalonbatch}
L_{\text{BYOL}, P^{*}} = \mathbb{E}\Big[ \left\| \operatorname{Proj}_{(Z_\theta)^{\bot}} (Z_{\xi}^{\prime}) \right\|_{F}^{2} \Big].
\end{equation}
The \byol objective for optimal linear predictor can be seen as a reconstruction loss, performing self-denoising. This is a specialization of the relationship $L_{\text{BYOL}, P^{*}} = \mathbb{E}\left[\sum_{j} \operatorname{Var}\left(z_{\xi, j}^{\prime} \mid z_{\theta}\right)\right]$ that appeared in \citet{Grill2020}. Our equation \ref{projectororthogonalonbatch} gives additional insights. 

First, it shows qualitatively how the stop-gradient operation is absolutely necessary: without stop-gradient, the projection argument $Z'_\xi$ would be a $Z'_\theta$, and \byol could be incentivized to minimize its loss by moving its representation $Z_\theta$ to be orthogonal to itself, triggering collapse. Second, we also see that the batch size controls the rank of denoising approximation when it is lower than feature dimensionality.

\textbf{Scaled orthogonal projection, towards the identity.} An orthogonal projection $P$ satisfies $P^2=P$ and $P=P^{\top}$. We will assume that the linear predictor is close to optimality, thus it verifies both equalities approximately during training, and exactly at the end of training. Furthermore, as the original normalized $L^2$ distance formulation of the \byol loss (with denominator) is invariant by strictly positive scaling of the linear predictor, we will also neglect scaling of the predictor from now on. Hence we can assume its maximal eigenvalue to be $1$ at all times by spectral normalization, and all of its eigenvalues to be close either to $1$ or $0$ (see Figures \ref{fig:DSVNS} and \ref{fig:DistribNS}), with the evolution of the predictor rank being a main object of interest. 
We now show that if both network branches in \byol fully co-adapt at the end of training, the optimal linear predictor becomes full-rank, i.e., the identity:
\begin{restatable}{prop}{propthree}\label{predictoridentity}
Assume a linear additive feature noise for each batch $Z_{\xi}^{\prime} = Z_{\theta} + \sigma_t \cdot G$, with G a random noise matrix whose entries we can hypothesize are i.i.d. with respect to each other and through time, and $\sigma_t$ a time-dependent, scalar variance chosen so that $|||G||| = 1$.  If additionally $\vertiii{Z_\theta - Z_\xi^{'}} = \sigma_t \rightarrow 0$ at convergence of \byol training, then the linear predictor converges to the identity matrix, asymptotically.
\end{restatable}
\emph{Proof:} From now on, our proofs are deferred to Appendix.

\subsection{Empirical evolution of predictor rank}
Proposition \ref{predictoridentity} shows the predictor can reach the full-rank identity matrix at the end of training, in the idealization that the training loss effectively goes to $0$. We thus seek to observe the evolution of the linear predictor rank during training. Since eigenvalues of small magnitude (compared to numerical precision) might make the rank hard to identify, we use proxy estimates. The \emph{stable rank} \citep{stablerank} of a matrix $M$---a lower bound on its standard rank---is defined as $\operatorname{srank}(M) = \left\|M \right\|_{F}^2/\vertiii{M}^2$, 
and for an orthogonal projection $P$ is just $\left\| P \right\|_{F}^2 = \operatorname{tr}(P P^{\top})=\operatorname{tr}(P)$. An alternative estimate for the rank of an approximate projection is also its trace. We empirically find that both rank metrics, stable rank and trace, increase during training (Figure \ref{fig:srankpred}), with the stable rank generally increasing \emph{monotonically}, showing expansion of the subspace span of the predictor. 
In particular, the eigenvalues of the rescaled predictor converge to either $0$ or $1$, with a greater fraction of the latter as training progresses. We can use this observation on the rank deficiency of the batchwise optimal predictor to derive a new closed-form predictor.

\begin{figure}[t]
    \centering
    \includegraphics[width=.495\textwidth]{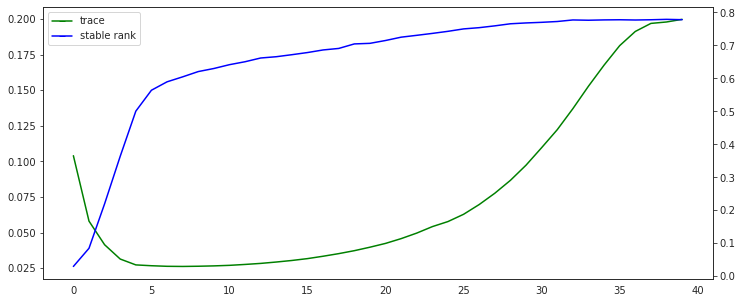}
    \vskip -0.75em
    \caption{Normalized stable rank (blue) and trace (green) of the \byol linear trainable predictor, throughout training on ImageNet-$1$k for $100$ epochs. Average over $3$ random seeds.}
    \label{fig:srankpred}
    \vskip -1em
\end{figure}

\subsection{A new closed-form predictor}

Intuitively the rank deficiency of the predictor early on in training---and potentially the associated batch covariance matrix $(Z_{\theta}^{\top} Z_{\theta})$---is one reason that can make the matrix inversion in the LRP predictor (Equation \ref{closed_form_lrp_predictor}) numerically unstable, and hinder performance. This suggests sidestepping this issue by using a Taylor-like expansion for the matrix inverse instead, leading to a new closed-form batchwise predictor:
\begin{restatable}{prop}{propfour}\label{NE}
The expression
\begin{equation}\label{neumann}
    P_{\theta} := 2 \cdot Z_{\theta}^{\top} Z_{\xi}^{\prime} - Z_{\theta}^{\top} Z_{\theta} Z_{\theta}^{\top} Z_{\xi}^{\prime}
\end{equation}
gives another batchwise, approximately optimal predictor. As it is obtained by Neumann expansion $\big(Z_\theta^{\top} Z_\theta \big)^{-1} \approx 2I - Z_\theta^{\top} Z_\theta + \cdots$, we name this the $\mathbf{NE}$ predictor.
\end{restatable}
If the magnitude of latent elements in batches is small, e.g. with well-chosen normalization, the fourth-order term $Z_{\theta}^{\top} Z_{\theta} Z_{\theta}^{\top} Z_{\xi}^{\prime}$ in Equation \ref{neumann} can be considered small, so that the optimal linear predictor reduces to $2 \cdot Z_{\theta}^{\top} Z_{\xi}^{\prime}$. Our NE predictor can therefore also be interpreted as generalizing the \emph{DirectCopy} \citep{Wang2021TowardsDR} predictor with an added correction term providing performance benefits, as we will see in Section~\ref{resultssection}.

\section{A spectral view of linear \byol}

In this section we take a tour of \byol through a spectral decomposition lens in order to characterize more precisely what it learns. We begin with showing that the optimal linear predictor, conditionally on the latents \emph{and its own rank} $k$, performs $k$-PCA on latents, identifying the leading $k$ eigendirections of the covariance of latents. We then prove that under additional linear features assumption, the linear encoder features are decorrelated at the end of training. This view justifies taking the matrix square root of the empirical covariance as a predictor; accordingly we propose two new predictor iterations.

\subsection{Spectral view of the rank-conditioned predictor}

The following proposition highlights what an optimal linear predictor learns, conditionally to its rank \emph{and} the latents coming from the encoder and projector networks:

\begin{restatable}{prop}{propfive}\label{prop-kpca}
Assume no stochastic augmentations or EMA operation ($\tau=1, \xi = \operatorname{sg}(\theta)$). Further, assume $P$ to be optimal and thus \textbf{exactly} an orthogonal projection batchwise. Conditionally on the latents \textbf{and} the rank of the predictor, $k$, the \byol inner loss solves a $k$-PCA problem. The predictor's span is then given by the top-$k$ eigenvectors of the latents' covariance matrix, $\mathbb{E}[Z^{\top}_{\theta} Z_\theta]$.
\end{restatable}

A linear optimal predictor $P^{*}_\theta$ in that restricted case performs $\operatorname{rank} P^{*}_\theta$-PCA on the covariance of latents, in particular finding mutually orthogonal eigenvectors. At the end of training, the predictor becomes full rank. This refines our understanding of the optimal linear predictor as an orthogonal projection, by characterizing the directions of its span.

\subsection{Decorrelation of linear encoder features}

The orthogonality of directions learned by the linear predictor extends to linear encoder features. At the end of training, the linear predictor is the identity (Proposition \ref{predictoridentity}). This fact can be leveraged thanks to the \emph{balancing relationship} \cite{Tian2021UnderstandingSL}, valid in the presence of a small weight decay parameter $\lambda$:
\begin{equation}\label{balancing}
A_{\theta,t} A_{\theta,t}^{\top} = P_{\theta,t}^{\top} P_{\theta,t} + M_0 e^{-2\lambda t}
\end{equation}
where $A_{\theta,t}$ and $P_{\theta,t}$ are now indexed by time $t$, and $M_0$ is a constant matrix determined at initialization. At the end of training the $M_0$ term has decayed exponentially and became negligible. In particular, since $P_{\theta,t}=I_f$ by then, we get that at the end of training, $A_{\theta,t} A_{\theta,t}^{\top} \rightarrow I_f$; encoder features are therefore decorrelated.

\subsection{Covariance square-root predictors} A method for performing non-contrastive self-supervised learning is to use a predictor that directly computes a \emph{matrix square root} of the covariance matrix of embeddings. This reflects the decorrelation of features seen above and was suggested in prior work, either using eigendecomposition \citep{Tian2021UnderstandingSL} or Cholesky factorization towards optimal whitening \citep{Ermolov2021}. It can be explained simply in our setting. Recall our batchwise NE predictor (Equation \ref{neumann}) can be, when neglecting higher-order terms, made directly proportional to the covariance of online and target features, $ P_\theta = Z_{\theta}^{\top} Z_{\xi}^{\prime} $. Since for an optimal, orthogonal projection, predictor $ P_\theta^2 = P_\theta$, we get  $ P_\theta^2 = Z_{\theta}^{\top} Z_{\xi}^{\prime}$, justifying using a square root of the covariance of features as an alternative predictor. We show how to do this efficiently.

\textbf{Approximate matrix square-root.}  Exact matrix square root methods through eigen- or Cholesky decomposition are computationally costly, potentially unstable, and do not exploit that the linear predictor only needs to be approximately optimal, as experiments in Appendix I of \citet{Grill2020} show. We therefore propose to use approximate matrix square root iterations, applied to the empirical covariance of \emph{online} latents $\Sigma = 1/b \cdot (Z_\theta^\top Z_\theta)$, \emph{in lieu} of a trainable linear predictor. Given batchwise covariance matrix $\Sigma$, these methods approach $\Sigma^{1/2}$, defined as $(\Sigma^{1/2})^{\top} \Sigma^{1/2} = \Sigma$, as iterations increase. We choose those iterations to be inverse-free, and involve exclusively matrix products and accumulation, so as to better leverage hardware accelerators, and be numerically stable. 

\textbf{Visser and Newton-Schulz iterations.} The two iterations we consider are the \emph{Visser} iteration \citep{Higham2008FunctionsOM} and the \emph{Newton-Schulz} iteration \citep{Schulz1933IterativeBD, Denman1976TheMS, Higham2008FunctionsOM}. We detail them below.
\begin{restatable}{prop}{propsix} \emph{(Visser)} \label{visser}
Set a maximum number $n$ of iterations, and a positive step size $\eta$. For each online latents batch $Z_\theta$, compute the matrix $\Sigma= \frac{1}{b}Z_\theta^{\top} Z_\theta$, and set $P_0 = \frac{1}{2 \eta} I_f$. Then repeat $n$ times the following \emph{Visser} iteration for the predictor $P_k$:
\vspace{-0.5em}
\begin{equation}
    P_{k+1} = P_k + \eta \cdot \big( \Sigma - P_{k}^2 \big)
\end{equation}
Then $P_k \underset{k}{\rightarrow} \Sigma^{1/2}= \frac{1}{\sqrt{b}}( Z_\theta^{\top} Z_\theta)^{1/2}$. Use $P_\theta:=P_n$ as batchwise predictor. We name this the \textbf{Visser} predictor.
\end{restatable}

In practice, the Visser predictor needs several dozens iterations to achieve good results, and requires choosing an additional step size parameter $\eta$. By contrast, the \emph{Newton-Schulz} iteration we next present does not have this requirement, and provides faster convergence as a second-order method.

\begin{restatable}{prop}{propseven} \emph{(Newton-Schulz)}\label{newtonschulz}
Set a maximum number $n$ of iterations. For each online latents batch $Z_\theta$, first set $\Sigma = \frac{1}{b} Z_{\theta}^{\top} Z_\theta, \quad A_0= \Sigma / \left\|\Sigma\right\|_F, \quad B_0 = I_f$, and then repeat $n$ times the coupled matrix iteration in $(A_k, B_k)$
\vspace{-0.5em}
\begin{empheq}[left=\empheqlbrace]{align}\label{nssystem}
A_{k+1} &= 1/2 \cdot A_{k} \big( 3I_f - B_k A_k \big) \\
B_{k+1} &= 1/2 \cdot \big( 3I_f - B_k A_k \big) B_{k}
\end{empheq}
Then $A_k \underset{k}{\rightarrow} A_{0}^{1/2} = (Z_{\theta}^{\top} Z_\theta)^{1/2} / \sqrt{b \left\|\Sigma\right\|_F}$. \\
Use $P_\theta:=A_n \cdot \sqrt{\left\|\Sigma\right\|_F}$ as batchwise predictor. \\ 
We name this the \textbf{Newton-Schulz} (\textbf{NS}) predictor.
\end{restatable}

We also note that iterate $B_k$ converges to a multiple of $\Sigma^{-1/2}$, which will be useful later.

This iteration only requires $3\cdot n$ matrix multiplications. Moreover, as a second-order method, it is robust to sensible choices of $n $. \citet{Higham2008FunctionsOM} discusses the theoretical and convergence properties of the two iterations in Propositions \ref{visser} and \ref{newtonschulz}. In our self-supervised learning context, these closed-form predictors recover and can even exceed the performance of a linear trainable predictor (cf. Section~\ref{resultssection}).

\textbf{$\text{NS}^2$}. Finally, given the projection relationship $P^2 \approx P$ for the approximately optimal linear predictor also implies that $P \approx P^{1/2} \approx P^{1/4} \approx \dots$, consistently with findings in \citet{Wang2021TowardsDR}, one can chain two successive applications of the Newton-Schulz iteration and use this as a predictor, now computing an approximation to $(Z_{\theta}^{\top} Z_\theta)^{1/4}$. In practice, we found this provides small performance gains. We call this method $\text{NS}^2$, and evaluate it empirically in Results section \ref{resultssection}. We however only chain Newton-Schulz iterations since we found that the chaining of two Visser iterations is prone to numerical collapse.

\section{A Riemannian interpretation of \byol}

This section views \byol and its induced orthogonality from a Riemannian geometry perspective. The optimal linear predictor is orthogonal, and can hence be seen as undergoing optimization on a \emph{Stiefel} manifold of orthogonal matrices, whose definition we recall below. We already know that conditionally on its rank $k$ and the latents, the optimal predictor solves a $k$-PCA problem (Proposition \ref{prop-kpca}). In the special case $k=1$, \byol can be seen as an exact instance of a $1$-PCA method called \emph{Krasulina}'s method. Extending this equivalence to $k \geq 1$ is possible with Riemannian interpretation. We show how orthogonality constraints in the predictor emerge from the architectural combination of symmetry and linear regression, corresponding to the notion of \emph{retraction} in Riemannian SGD. It also corresponds to the addition of negative gradients, precluding training collapse. We also derive another predictor form from those insights.

\subsection{$1$-PCA}

\textbf{Connection with Oja's and Krasulina's methods.} When trying to perform whitening or PCA, for large datasets, the sample covariance of data might differ from the true, population covariance; also, data may only be available incrementally. Algorithms have been devised to deal with this streaming setting. \emph{Oja's rule}~\citep{Oja1982SimplifiedNM, Oja1992PrincipalCM}, maintains an estimate $p_t$ of the \emph{leading} eigenvector of the covariance $\mathbb{E}[X^{\top} X]$, as given a realization of single random row vector $X$, it alternates between a gradient ascent update towards $X^{\top} X$ and a normalization update for $p_t$ according to
\vspace{-0.5em}
\begin{equation}\label{defn_oja}
p_{t} \leftarrow p_{t-1}+\eta\left( X^{\top} X p_{t-1} \right) \text { and } p_{t} \leftarrow \frac{p_{t}}{\left\|p_{t}\right\|}
\end{equation}
\emph{Krasulina's method}~\citep{Krasulina1969TheMO} performs a similar update, that folds $w_t$ normalization into an additional term:
\vspace{-0.5em}
\begin{equation}\label{defn_krasulina}
\small
p_{t} \leftarrow p_{t-1}+\eta\left(X^{\top} X p_{t-1 }- p_{t-1}\left(X \frac{p_{t-1}}{\left\|p_{t-1}\right\|}\right)^{2} \right)
\end{equation}
Both methods are known to perform $1$-PCA in the limit, and their convergence properties and rates are well studied~\citep{Shamir2014ASP, Jain2016StreamingPM}. In the \byol case, this gives us an analogy when the linear predictor is constrained to be exactly of rank $1$, showing it learns the leading eigendirection first:

\begin{restatable}{prop}{propeight}\label{krasulinaequiv}
Assume no stochastic augmentations or EMA operation ($\tau=1, \xi = \operatorname{sg}{\theta}$). The \byol predictor optimization loop conditional on latents, when constrained to operate with a rank-$1$ predictor, performs exactly Krasulina's update.
\end{restatable}
\vspace{-0.5em}
Proposition \ref{krasulinaequiv} strengthens the Proposition \ref{prop-kpca} in the special case $k=1$. The case $k=1$ is illuminating but begs for generalization to larger ranks. We now proceed to extend this to $k \geq 1$.

\subsection{A Riemannian interpretation}
The Oja and Krasulina updates can be adapted to $k \geq 1$, now maintaining a rank $k$ matrix estimate $P_t$, provided the explicit $p_t$ vector normalization steps in Equation \ref{defn_oja} are replaced with an \emph{orthonormalization step} for matrix $P_t$. This can be achieved by various means, including taking the orthogonal part in QR factorization of $P_t$ after each update, and ensures the procedure does not collapse. \citet{Tang2019ExponentiallyCS} for instance propose a matrix version of the Krasulina procedure this way.

\textbf{Riemannian SGD.} This general procedure of taking gradient updates towards an empirical covariance matrix, interleaved with orthonormalization steps, can be interpreted in another way. If it is acceptable to make the orthonormalization procedure approximate, we can see this as an instance of stochastic \emph{Riemannian gradient ascent} \citep{Bonnabel2013StochasticGD} as follows. Enforcing orthonormality constraints in matrix optimization is equivalent to solving an unconstrained problem over a \emph{Stiefel} matrix manifold \citep{Edelman1998TheGO, Absil2007OptimizationAO} defined for integers $(n,p)$, $p \leq n$ as the manifold $\operatorname{St}(n, p)$, of matrices $M$ such that $\left\{M \in \mathbb{R}^{n \times p}: M^{\top} M=I_p \right\}$. A point on the manifold is thus a set of $p$ orthonormal vectors in $\mathbb{R}^{n}$ (a $p$-frame) and can be seen as the given of a dimension $p$ subspace, as well as a choice of its basis vectors.

Riemannian gradient ascent proceeds by taking gradient steps, projected in the tangent space of the manifold at each step. However, because the gradient iterates remain in the tangent plane as they neglect curvature terms, they need to be systematically taken back (arbitrarily close) to the manifold itself through an operation called \emph{retraction}, $\mathcal{R}$, after each step. This results in a sequence of matrix iterates $(M_k)$ of the form, given step size $\eta$, and gradients of function $F$, 
\begin{equation}\label{retraction}
    M_{k+1} = \mathcal{R} \big(M_k, \eta \operatorname{Grad} F(M_k) \big)
\end{equation}
Retractions are generally computationally expensive matrix functions \citep{Ablin2022FastAA}, and a large part of matrix manifold optimization consists in devising efficient retraction algorithms~\citep{Absil2007OptimizationAO}.

\textbf{EMA and linear predictor as trivial retraction.} Recall that the optimal linear predictor in \byol is an orthogonal projection (Equation \ref{projectororthogonalonbatch}). We link this to the \byol architecture with elementary arguments. A projection is an idempotent linear application $P$, such that $P^2 - P = 0$, and an orthogonal projection satisfies $P P^{\top} - P = 0$, so a symmetric ($P = P^\top$) projection is also orthogonal. These properties can be made approximate given some scalar $\varepsilon \geq 0$; if defining an $\varepsilon$-\emph{quasi}-projection as satisifying $\vertiii{P^2 - P} \leq \varepsilon$, an $\varepsilon$-quasi-symmetry $\vertiii{P^{\top}-P} \leq \varepsilon$, and an $\varepsilon$-\emph{quasi}-orthogonal projection $\vertiii{P P^{\top} -P} \leq \varepsilon$. An $\varepsilon$-symmetric $\varepsilon$-projection is thus a ($2 \cdot )\varepsilon$-orthogonal projection. 

We now observe crucially that a close-to-optimal linear regression predictor in \byol should be close to a projection, and an exponential moving average operation on the representations with $\tau \simeq 1$ should yield a quasi-symmetric predictor. Importantly the EMA update step is done after the stop-gradient operation. The interleaved gradient updates-EMA steps in \byol are thus analogous to gradient updates-retraction steps in Riemannian SGD. Hence,

\begin{centering}
A $\underbrace{\text{quasi-symmetric}}_{\text{tight EMA}}$ $\underbrace{\text{quasi-projection}}_{\text{linear regression}}$ is $\underbrace{\text{quasi-orthogonal}}_{\text{Stiefel retraction}}$. \\
\end{centering}
\vspace{0.5em}

i.e. the coupling of the EMA and predictor make them an approximate orthonormalization step for the latents; \textbf{given a close to optimal linear predictor, the additional EMA update makes it a computationally free Stiefel retraction, on $\operatorname{St}(f, \operatorname{rank} P)$.}

\begin{figure}[ht]
    \centering
    \includegraphics[keepaspectratio,width=.45\textwidth]{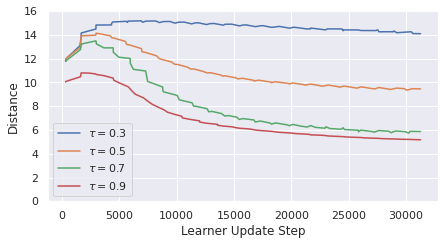}
    \caption{Distance of the linear predictor to its polar factor, as a function of the EMA parameter $\tau$ in \byol. Smaller EMA values tend to result in higher terminal distances. ImageNet, $100$ epochs.}
    \label{fig:EMAandpolar}
\end{figure}

We empirically verify this analogy by observing (Figure \ref{fig:EMAandpolar}) that when the EMA parameter in \byol is decreased (down to $0$), then at the end of training, the distance between the linear predictor $P$ and its unitary, polar factor $U$ in polar decomposition $P = U H$, fails to decrease. The EMA operation enables positioning close to the Stiefel manifold early on in training. Our view thus replaces and refines the interpretation of \byol and \simsiam as expectation-maximization algorithms~\citep{SimSIAM}.

\textbf{Towards a family of orthonormalizing SSL methods.} This analogy would also enable us to replace the joint operation of a learnable predictor and an EMA in \byol with any known closed-form Stiefel retraction, or outright projection, operating on the covariance of latents, deriving multiple self-predictive unsupervised learning methods at will. This uncovers a dimension of trading off complexity for number of learning epochs at given terminal performance, opening possibilities for leaner, sample-efficient self-supervised learning. We propose the simplest of these methods below.

\textbf{Stiefel projection.} For any matrix $M$ (square or not) it is known that the matrix $(M^\top M)^{-1/2} M^\top$ represents the projection of $M$ on the Stiefel manifold. We can use an approximation to this expression instead of a retraction, applied to the covariance of online features to orthogonalize it, as a closed-form predictor. Since the sequence of matrices $(B_k)$ in the Newton-Schulz iteration (Proposition \ref{newtonschulz}) allows estimating the \emph{inverse} matrix square root $\Sigma^{-1/2}$, we have:

\begin{restatable}{prop}{propten}\label{stiefelpredictor}
Set a maximum number $n$ of iterations. For each online latents batch $Z_\theta$, first compute the matrix $\Sigma=\frac{1}{b}Z_\theta^{\top} Z_\theta$. Then compute an estimate to $(\Sigma^\top \Sigma)^{-1/2} \Sigma^\top$ using $B_n$ after applying $n$ Newton-Schulz iterations to $(\Sigma^{\top} \Sigma)$, as per Proposition \ref{newtonschulz}. Use $P_n:=B_n \Sigma^{\top} / \left\|B_n\right\|_F $ as a batchwise predictor. We name this the \textbf{Stiefel} predictor.
\end{restatable}

\begin{figure*}
     \centering
     \begin{subfigure}[t]{0.32\textwidth}
        \centering
        \includegraphics[width=\textwidth]{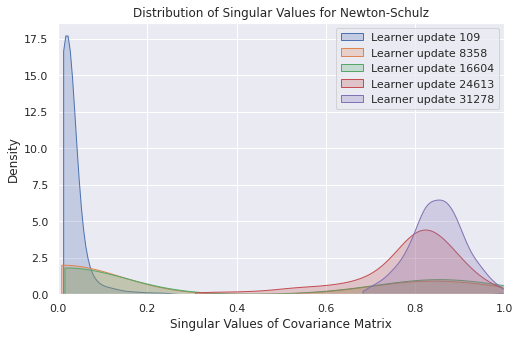}
        \caption{Empirical distribution of the singular values of the covariance of features, for Newton-Schulz, throughout $100$ epochs.}
        \label{fig:DSVNS}
    \end{subfigure}
    \hfill
    \begin{subfigure}[t]{0.32\textwidth}
        \centering
        \includegraphics[width=\textwidth]{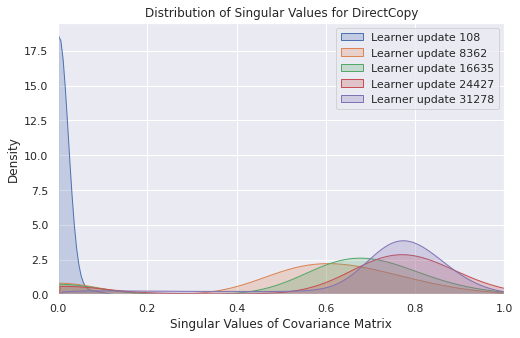}
        \caption{Empirical distribution of the singular values of the covariance of features, for DirectCopy, throughout $100$ epochs.}
        \label{fig:DSVDirectCopy}
    \end{subfigure}
    \hfill
    \begin{subfigure}[t]{0.32\textwidth}
        \includegraphics[width=\textwidth]{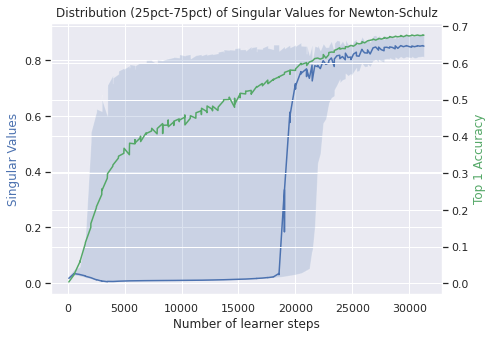}
        \caption{\emph{Blue:} Median, $25$th and $75$th percentile of empirical distribution as in leftmost Figure \ref{fig:DSVNS}. \emph{Green:} Top-$1$ $\%$ linear accuracy.}
        \label{fig:DistribNS}
    \end{subfigure} 
    \caption{Study of the (dynamically renormalized) singular values of the covariance of online latents.}
    \label{fig:distrib_main}
\end{figure*}

\textbf{Riemannian SGD can create negative gradients.} Finally we show how the Riemannian perspective helps explain non-collapse in negative-free SSL. Performing Riemannian gradient descent (for the predictor) is done by first computing \emph{Riemannian} gradients. Those are standard gradients projected onto the tangent plane to the manifold at every iterate. Riemannian gradients on the Stiefel manifold contain an additional term relative to standard gradients. The following proposition is given in \citet{Edelman1998TheGO}:
\begin{restatable}{prop}{propnine}\label{StiefelGradients}
Let $F$ be a function of a matrix $P$ and denote by $F'(P)_{ij}= (\frac{\partial F}{\partial P_{ij}})$ the Jacobian (matrix gradient) of $F$ with respect to $P$. The Riemannian gradient of $F(P)$ on the Stiefel manifold, $\operatorname{Grad} F(P)$, is given by
\begin{equation*}
    \operatorname{Grad} F(P) = F'(P) - P F'(P)^{\top} P.
\end{equation*}
\end{restatable}
The additional term $- P F'(P)^{\top} P$ can thus possibly play the same role as standard negative terms intervening explicitly in contrastive losses. This helps provide intuition on why orthonormality constraints on the predictor can prevent its collapse, and in turn, any collapse of features due to eigen-alignment.

\textbf{Summary of our closed-form predictors.} For clarity, we show a summary of the batchwise predictors we propose Table \ref{summarypredictors}. We then proceed to evaluate these empirically.

\vspace{-1.0mm}
\begin{table}[!ht]
\skip -0.5em
\centering
\small
\begin{tabular}{l|c}
Closed-form batchwise predictor & Computes $P_\theta$ as \\
\midrule
LRP \citep{Grill2020}   &    $\left(Z_{\theta}^{\top} Z_{\theta}\right)^{-1} Z_{\theta}^{\top}  Z_{\xi}^{\prime}$     \\
DirectPred \citep{Tian2021UnderstandingSL} & $(Z_{\theta}^{\top}Z_{\theta})^{1/2}$ (eigendecomp.) \\
DirectCopy \citep{Wang2021TowardsDR} & $Z_{\theta}^{\top}Z_{\theta}$ \\
\midrule
NE (Prop. \ref{NE})                 & $2 \cdot Z_{\theta}^{\top} Z_{\xi}^{\prime} - Z_{\theta}^{\top} Z_{\theta} Z_{\theta}^{\top} Z_{\xi}^{\prime}$         \\
Visser (fast iteration) (Prop. \ref{visser})              & $(Z_{\theta}^{\top}Z_{\theta})^{1/2}$         \\
NS (fast iteration) (Prop. \ref{newtonschulz})                  &  $(Z_{\theta}^{\top}Z_{\theta})^{1/2}$        \\
Stiefel (fast iteration) (Prop. \ref{stiefelpredictor})    & $(Z_{\theta}^{\top }Z_{\theta}Z_{\theta}^{\top }Z_{\theta})^{-1/2} (Z_{\theta}^{\top}Z_{\theta})$   
\end{tabular}
\skip -0.5em
\caption{\label{summarypredictors}Summary of closed form predictors}
\end{table}

\vspace{-1.0em}
\section{Results}\label{resultssection}

\textbf{Experimental protocol.} In this section we follow the experimental protocol of \byol in \citet{Grill2020}, except we use a linear predictor $P$, taken as a closed-form expression defined batchwise (see Table \ref{summarypredictors}), instead of a trainable one. Importantly, we use an EMA on the predictor to stabilize its numerics. This aside, we use the same ResNet-$50$ encoder with a $2$-layers projector, and optimization settings (including learning rate and weight decay) as in \citet{Grill2020}. We report top-$1$ percentage accuracy on ImageNet-$1$k under linear evaluation protocol, after either $100$ or $300$ epochs, averaged over $3$ random seeds.
More details and hyperparameters are presented in Appendix~\ref{appx:hyper}.

\textbf{Predictor Ridge.} When estimating the matrix predictor through closed-form expressions or iterations, it is possible to add an extra regularization or \emph{ridging} term. It consists in adding $\alpha$ times the identity matrix to the predictor. First, it provides additional numerical stability. Second, it brings the predictor \emph{closer} to the orthogonal manifold, which we have shown to be an important property in the predictor. We report our results with $\alpha$ values between $+0.0$ and $+0.9$ in $+0.15$ increments (Tables \ref{hundredepochs} and \ref{threehundredepochs}).

\textbf{\byol baseline and main results.} The top-$1$ accuracy reached by standard \byol with a linear \emph{trainable} predictor is $67.6\%$ at $100$ epochs and $71.5\%$ at $300$ epochs. Our best results achieved with closed-form predictors (with additional ridge) are $68.8\%$ and $72.2\%$ respectively, with both the NS predictor (in its $\text{NS}^2$ variant) and the Stiefel predictor actually outperforming the linear trainable variant, by $+1.2\%$ and $+0.7\%$ respectively. At $300$ epochs, the performance of the NS and Stiefel predictors remains stable across a broad range of ridge parameters $\alpha$ as soon as $\alpha > 0$, removing the need to tune it precisely. We show our performance at $100$ epochs versus prior art in Table~\ref{hundredvscomp}. We reach the best performance across the board with $68.8\%,$ along with \emph{DirectCopy} for a specific ridge value of $\alpha=0.01$, speaking to the necessity of tuning that parameter very precisely there. 
By contrast, adding a corrective term via our NE predictor is enough to improve performance in a wide $\alpha$ range compared to DirectCopy, and function well even without ridging, or tuning predictor setting frequency. Our $\text{NS}^2$ and Stiefel predictors perform the best both at $100$ and $300$ epochs.
Overall, these positive results of our closed-form predictors empirically validate our analyses for each form of the \byol predictor.

\begin{table}[ht]
\centering
\scalebox{0.85}{
\begin{tabular}{l|cccccccc}
Predictor   & $0.0$ & $0.15$ & $0.3$ & $0.45$ & $0.6$ & $0.75$ & $0.9$  \\
\midrule
$\text{DirectCopy}^{\dagger}$  & $20.9$&  $65.8$& $67.3$&  $67.4$& $67.5$& $67.4$ & $67.4$   \\ \midrule
NE         & $61.2$&  $\underline{68.5}$& $68.3$&  $68.1$& $67.7$&$67.4$  &  $67.0$ \\
Visser & $\underline{68.4}$ & $68.3$ & $68.4$ & $68.4$ & $68.3$ & $68.4$ & $68.3$ \\
NS ($n=9$)        & $66.4$&  $66.4$& $66.7$&  $67.1$& $67.6$& $67.8$ & $\underline{68.0}$  \\
$\text{NS}^{2}$ ($n=7$) & $66.8$ & $68.2$ & $68.6$ & $68.6$ & $68.7$ & $68.7$ & $\underline{\mathbf{68.8}}$ \\
Stiefel ($n=9$) & $66.7$ & $67.7$ & $\underline{\mathbf{68.8}}$ & $68.7$ & $68.7$ & $68.7$ & $68.7$ \\ \bottomrule
\end{tabular}
}
\vspace{-3.0mm}
\caption{\label{hundredepochs}Top-$1$ accuracy on ImageNet-$1$k (linear protocol) of our predictors, at $100$ epochs, as ridge $\alpha$ varies. \emph{DirectCopy} \citep{Wang2021TowardsDR} is our own replication, setting predictor every batch. Best results per predictor underlined, best overall bolded. Average over $3$ random seeds. Standard trainable linear predictor \byol achieves $67.6\%$.}
\end{table}

\vspace{-3.0mm}
\begin{table}[ht]
\centering
\scalebox{0.85}{
\begin{tabular}{l|cccccccc}
Predictor   & $0.0$ & $0.15$ & $0.3$ & $0.45$ & $0.6$ & $0.75$ & $0.9$  \\
\midrule
$\text{DirectCopy}^{\dagger}$  & $48.3$& $71.1$ & $71.4$& $71.3$ & $71.2$& $71.5$ & $71.4$    \\ \midrule
NE         & $68.8$& $\underline{71.8}$&  $71.7$& $71.5$ & $71.3$& $70.9$ &  $70.6$  \\
Visser & $\underline{71.6}$ & $71.5$ & $71.6$ &$71.4$ &$71.5$ &$71.4$ &$71.6$ \\
NS $(n=9)$         &$70.8$ &  $71.7$& $\underline{72.0}$&$72.0$  & $71.9$& $72.0$ & $71.9$   \\
$\text{NS}^{2}$ $(n=7)$ & $70.9$ & $72.0$ & $72.1$ & $72.0$ & $\underline{\mathbf{72.2}}$ & $72.1$ & $72.1$ \\
Stiefel $(n=9)$ & $71.4$ & $72.0$ & $\underline{72.1}$ & $72.1$ & $72.0$ & $72.0$ & $72.1$ \\ \bottomrule
\end{tabular}
}
\caption{\label{threehundredepochs}Top-$1$ accuracy on ImageNet-$1$k (linear protocol) of our predictors, at $300$ epochs, as ridging $\alpha$ varies. Average over $3$ seeds. Standard trainable linear predictor \byol achieves $71.5\%$.}
\vspace{-1.0em}
\end{table}

\begin{table}[!h!b]
\skip -0.5em
\centering
\small
\scalebox{0.875}{
\begin{tabular}{l|c}
Method                                  & Accuracy \\ \midrule
SimSiam \citep{SimSIAM}                 & $68.1$                \\
Barlow Twins \citep{Zbontar2021BarlowTS} & $63.4$                \\
VICReg \citep{Bardes2022VICRegVR}       & $65.3$                \\
\citet{bridgingthegap} (no covariance partitioning) & $61.4$                \\
\citet{bridgingthegap} (covariance partitioning)    & $67.3$                \\
DirectPred \citep{Tian2021UnderstandingSL} & $68.5$                \\
DirectCopy \citep{Wang2021TowardsDR}$, \alpha = 0.01 $             & $\bf{68.8}$                \\
Neural Eigenmap \citep{Deng2022NeuralEA}& $68.4$\\ \midrule
Visser (ours)                 & $68.4$                \\
$\text{NS}^{2}$  (ours)      & $\bf{68.8}$      \\
Stiefel  (ours)      & $\bf{68.8}$      \\\bottomrule        
\end{tabular}
}
\caption{\label{hundredvscomp}Top-$1$ accuracy on ImageNet-$1$k (linear protocol), at $100$ epochs, comparing prior art (with standard, rather than higher capacity, projector network architecture) and our best results.}
\end{table}

\textbf{Performance of square-root iterations.} The Visser, NS, and Stiefel predictors all perform better than the \byol baseline. Consistently with its initialization as a large multiple of the identity matrix, the Visser iteration is insensitive to ridging, and can be used with $\alpha=0$. We also evaluate the $\text{NS}^2$ variant of the Newton-Schulz iteration, chaining two of its successive applications, but with two times $7$ iterations instead of $9$ for the single iteration. We see performance gains from this approach ($68.8\%$ instead of $68.0\%$ at $100$ epochs) and hypothesize that they stem from better pre-conditioning of the covariance of features, yielding richer eigendirections to explore. Finally, the Stiefel predictor shows excellent performance both at $100$ and $300$ epochs.

\textbf{Predictor singular values.} We observe the empirical distribution of singular values of the latents' covariance throughout $100$ epochs of training. We continously normalize those to be between $0$ and $1$, and observe an average of $3$ seeds. We do this across two predictors: Newton-Schulz (Figures \ref{fig:DSVNS} and \ref{fig:DistribNS}), and DirectCopy \citep{Wang2021TowardsDR} (Figure \ref{fig:DSVDirectCopy}). Empirically we observe better separation of eigenvalues with our Newton-Schulz predictor, consistently with the robustness and performance differentials highlighted above. Additional depictions for our other predictors and ablation studies on the number of iterations $n$ for our approximate methods are provided in Appendix~\ref{appx:ablation}. The robustness of the Stiefel predictor and its strong performance across epochs and $\alpha$ makes it the best method to use, in our view.

\section{Related work and conclusion}\label{sec:related_work}

We have sought to present minimum theoretical arguments towards understanding self-predictive unsupervised methods, using as few mathematical assumptions as possible. Previous attempts at such explanation have involved implicit contrastive properties of normalizations later disproved empirically \citep{Richemond2020BYOLWE}, or also relied on strong assumptions such as isotropy and multiplicative EMA \citep{Tian2021UnderstandingSL} and on differential equation tools for their conclusions.

Instead, we focused on understanding the multiple forms of the \byol predictor network, forming a bridge to existing methods (Table \ref{summarypredictors}). The proximity of the linear predictor to an orthogonal projection (with increasing subspace span throughout training, uncovering more eigendirections of the covariance of latents until the identity matrix is reached) is a unifying insight.
This helps explain existing methods with very simple algebraic arguments and allows us to derive new ones.
Therefore we introduce multiple closed-form predictor variants at once, in contrast with prior work 
\citep{Bardes2022VICRegVR, Zbontar2021BarlowTS, bridgingthegap}. Our Riemannian manifold perspective complements that of \citet{Balestriero2022ContrastiveAN} and also enables deriving additional methods at will, through Stiefel retractions. We operationalized our understanding with efficient implementations of our predictors that utilize fast matrix iterations, new to the self-supervised learning context. These closed-form predictors perform strongly, offer better insights on the representations learned by self-supervised methods, and help further explain the importance of the predictor network.

\nocite{latestICLRWang, weng2022investigation, zhang2022zerocl}
\nocite{SpectralNet, Pfau2019SpectralIN, Deng2022NeuralEA}

\clearpage
\bibliography{main}
\bibliographystyle{plainnat}

\clearpage
\onecolumn

\begin{appendix}

\section*{\centering Supplementary Material}

\section{Proof of theoretical results}

\propone*
\begin{proof}
Recall that we have on each batch
\begin{equation}
P^{\star}=\underset{P}{\arg \min } \quad \left\|Z_{\theta} P - Z_{\xi}^{\prime}\right\|_{F}^{2} = \left(Z_{\theta}^{\top} Z_{\theta}\right)^{-1} Z_{\theta}^{\top}  Z_{\xi}^{\prime}
\end{equation}
We can just insert the value of the optimum in the \byol objective (right and left term of equation above) yielding
$$\left\|Z_{\theta} \left(Z_{\theta}^{\top} Z_{\theta}\right)^{-1} Z_{\theta}^{\top}  Z_{\xi}^{\prime} - Z_{\xi}^{\prime} \right\|_{F}^{2} = \left\| \Big(  Z_{\theta} \left(Z_{\theta}^{\top} Z_{\theta}\right)^{-1} Z_{\theta}^{\top} - I_f \Big) Z_{\xi}^{\prime}   \right\|_{F}^{2}$$

Here $Z_{\theta} \left(Z_{\theta}^{\top} Z_{\theta}\right)^{-1} Z_{\theta}^{\top}$ is recognized as the linear, orthogonal projection application on the vector space spanned by the rows of batch $Z_\theta$, which is of dimension at most $\min (b,f)$; we denote this projection by $\operatorname{Proj}_{(Z_\theta)}$. Its complement to the identity is also an orthogonal projection, $\operatorname{Proj}_{(Z_\theta)^{\bot}}$.
Thus for an \emph{optimal} linear predictor, on each batch,
\begin{equation}\label{orthproj}
L_{\text{BYOL}, P^{*}} = \mathbb{E}\Big[ \left\| \operatorname{Proj}_{(Z_\theta)^{\bot}} (Z_{\xi}^{\prime}) \right\|_{F}^{2} \Big]
\end{equation}
\end{proof}
One added benefit of the orthogonal projector representation, see Eq.~\eqref{projectororthogonalonbatch}, is that it enables identifying two sources of noise intervening in \byol thanks to decomposing $Z_{\xi}^{\prime}$. We additionally have the proposition:

\begin{proposition*}\label{prop:EMAasregularizer}
 The \byol loss with optimal linear predictor decomposes into two noise terms: resampling noise and EMA noise.
 \begin{equation*}
L_{\text{BYOL}, P^{*}} = \mathbb{E}\left[ \left\| \operatorname{Proj}_{(Z_\theta)^{\bot}} \big[ \underset{\text{resampling noise}}{(Z_{\xi}^{\prime} - Z_\xi)} + \underset{\text{EMA noise}}{(Z_\xi - Z_\theta)} \big] \right\|_{F}^{2} \right].
\end{equation*}
\end{proposition*}

\begin{proof}
By rewriting $Z_{\xi}^{\prime} = (Z_{\xi}^{\prime} - Z_\xi) + (Z_\xi - Z_\theta) + Z_\theta$ and substituting this expression into Equation \ref{orthproj}, where $Z_\theta$ then vanishes. $Z_{\xi}^{\prime} - Z_\xi$ can be interpreted as resampling noise due to stochastic augmentations, while the difference $(Z_\xi - Z_\theta) + Z_\theta$ can be seen as exponential moving average noise.

This decomposition may help explain why the performance of tied-weights (that is, EMA-less, with $\tau=1$) versions of \byol \citep{Grill2020, SimSIAM} is suboptimal: the vanishing of the second term providing an additional source of gradient for learning.
\end{proof}

\propthree*
\begin{proof} Substituting $Z_{\xi}^{\prime} = Z_{\theta} + \sigma_t \cdot G$ into $P^{*}_\theta = \left(Z_{\theta}^{\top} Z_{\theta}\right)^{-1} Z_{\theta}^{\top}  Z_{\xi}^{\prime}$ (Equation  \ref{closed_form_lrp_predictor}) immediately yields:

\begin{equation*}\label{optlinpredG}
P^{*}_\theta = I_f + \sigma_t \left(Z_{\theta}^{\top} Z_{\theta}\right)^{-1} Z_{\theta}^{\top} G.
\end{equation*}
and then by assumption $\sigma_t \rightarrow 0$ which makes the second term vanish (provided for instance features $Z_\theta$ are bounded, which is realized in practice) and gets the result we want of $P^{*}_\theta \rightarrow I_f$. \\

We can prove a slightly stronger result in this model. In this idealized setting $P^{*}_\theta$ cannot be exactly an orthogonal projector batchwise (as $P^{*2}_\theta \neq P^{*}_\theta$), but we get a formal control on how much it deviates from that property by upper-bounding $P_\theta^{*} P^{*\top}_\theta - P_\theta^{*}$ and writing 
\begin{equation*}\label{elaboratecontrol}
    \frac{\vertiii{P_\theta^{*} P^{*\top}_\theta - P_\theta^{*}}}{\vertiii{P_\theta^{*}}} \leq \sigma_t \vertiii{Z_\theta}^{-1} \rightarrow 0 \quad \text{if $Z_\theta$ bounded.}
\end{equation*}
\end{proof}

\propfour*
\begin{proof}
We obtain this thanks to $\big(Z_\theta^{\top} Z_\theta \big)^{-1} \approx 2I - Z_\theta^{\top} Z_\theta + \cdots$, the leading term of a \emph{Neumann expansion} \citep{Zhu2015OnTM}. This is simply substitued into the optimal linear regression predictor $\left(Z_{\theta}^{\top} Z_{\theta}\right)^{-1} Z_{\theta}^{\top}  Z_{\xi}^{\prime}$ (Equation \ref{closed_form_lrp_predictor}). Using only the first term is not only justified at the end of training (when $\left( Z_{\theta}^{\top} Z_{\theta}\right)^{-1} \rightarrow I_f$) but also if we can scale the norm of batches $Z_\theta$ to be small. This generalizes the usage of $P^{*}_{\theta} \approx \mathbb{E} \big[ Z_{\theta}^{\top} Z_{\xi}^{\prime} \big]$, which was separately named \emph{DirectCopy} in \citep{Wang2021TowardsDR}.
\end{proof}

\propfive*
\begin{proof}

Since $P = P^{\top} P$ as it is exactly an orthogonal projection, we can rewrite the \byol loss minimization as
\begin{align}\label{kpca}
    \min _{P, P P^{\top}=I_{k}} \mathbb{E}\left\|Z_\theta - P Z_{\operatorname{sg}(\theta)} \right\|^{2} &= \min _{P , P P^{\top}=I_{k}} \mathbb{E}\left\|Z_{\theta}-P^{\top} P Z_{\operatorname{sg}(\theta)}\right\|^{2} \\ 
    &= \min _{P , P P^{\top}=I_{k}} \operatorname{tr}\left(\left(\mathrm{I}_{d}-\mathrm{P}\right) \mathbb{E}\left[\mathrm{Z}_{\theta} \mathrm{Z}_{\operatorname{sg}(\theta)}^{\top} \right]\right)
\end{align}
which directly reflects the variational characterization of $k$-PCA.

We also note that this last form sheds light on a surprising empirical observation that parametrizing the linear predictor to be \emph{residual} \citep{ResNet}, effectively initializing $P$ at the identity, precludes all learning and makes no progress whatsoever, despite it being mathematically equivalent to the standard formulation of a neural predictor.
\end{proof}

\propeight*
\begin{proof}
Simply recovered by writing the \byol inner objective (with rank-$1$ predictor parametrized by row vector $p$, $L^2$ normalized  without loss of generality), so that $P = \frac{p^{\top} p}{\left\|p\right\|^{2}}$:
\begin{equation}
    \mathcal{L}_{\text{BYOL}} = \mathbb{E}\left\|Z_{\theta}-\frac{p^{\top} p}{\left\|p\right\|^{2}} Z_{\operatorname{sg}(\theta)}\right\|^{2}
\end{equation}
and then computing $\nabla_p \mathcal{L}_{\text{BYOL}}$ for this expression, then taking a gradient ascent step $w_t \leftarrow w_{t-1} + \epsilon \cdot \nabla_p \mathcal{L}_{\text{BYOL}}$ (see \citet{Tang2019ExponentiallyCS}).
\end{proof}

\newpage
\section{Details of algorithms and pseudocode}

\subsection{Hyperparameters}\label{appx:hyper}

Our data augmentations and evaluation protocols rigorously follow the standards described in~\citep{Grill2020, SimCLR}. Similarly, our hyperparameters are the same as in~\citet{Grill2020}, except for two differences. The first one is the learning rate and weight decay combination we employ - given shorter training compared to the original $1000$ epochs used in \byol, we choose a LARS~\citep{You2017LargeBT} optimizer initial learning rate of $0.45$ at $100$ epochs and $0.3$ at $300$ epochs, both in conjunction with a weight decay parameter of $1.0 \times 10^{-6}$ and an initial target decay rate $\tau$ set to $0.99$.

The second is the addition of an exponential moving average operation applied to the predictor matrix after computation of the relevant predictor formula (and potential extra ridging operation). We found empirically that this EMA operation stabilizes training. The formula we use to update (previous) predictor $P$ given current batchwise iterate $P_n$ is
\begin{equation}\label{predEMA}
 P \leftarrow \rho P + (1-\rho) P_n
\end{equation}

The value of constant covariance EMA parameter $\rho$ is predictor-dependent, and picked at $0.8$ for the LRP predictor, $0.99$ for the NE, Visser, NS and $\texttt{NS}^2$ predictors, and $0.999$ for the Stiefel predictor. We did not seek to optimize for a specific temporal schedule for $\rho$, which might enable unifying its value across all predictors. 

Finally, we set the Visser step size ($\eta$ in Proposition \ref{visser}) to $0.001$.

\subsection{Predictors pseudo-code}

For ease of reproducibility and clarity of understanding, we below give Python $3$ and specifically JAX \citep{jax2018github} pseudo-code that can be used to compute our predictors, reflecting main text equations, as well as any additional scaling factors that we empirically found could help numerical stability (at the margin). Notation-wise, the two lines
\noindent\begin{minipage}[t]{0.95\textwidth}
  \begin{algorithm}[H]
\begin{lstlisting}[language=Python, label={alg:latents}]
z_theta = online_out['projection_view1']
z_xi = target_out['projection_view2']
\end{lstlisting}
\end{algorithm}
\end{minipage}%

correspond to the computation of batches of latents $Z_\theta$ and $Z^{'}_{\xi}$, respectively. We begin with our $4$ proposed formulaic predictors, before also giving, for convenience, a comparable way of computing the LRP predictor from~\citep{Grill2020}.

\subsubsection{NE predictor pseudo-code}

\noindent\begin{minipage}[t]{0.95\textwidth}
  \begin{algorithm}[H]
\begin{lstlisting}[language=Python, label={alg:NE}, caption=Pseudo-code for the NE predictor.]
def ne_predictor(zt, zx):
  
  zttzx = jax.numpy.matmul(jax.numpy.transpose(zt), zx)
  p = 2.0 * zttzx
  p -= jax.numpy.matmul( jax.numpy.transpose(zt), jax.numpy.matmul(zt, zttxt))
      
  return p
  
z_theta = online_out['projection_view1']
z_xi = target_out['projection_view2']

z_theta /= jax.numpy.linalg.norm(z_theta, ord=2)
z_xi  /= jax.numpy.linalg.norm(z_xi, ord=2)

pred = ne_predictor(z_theta, z_xi)
pred /= jax.numpy.linalg.norm(pred, ord=2)

\end{lstlisting}
\end{algorithm}
\end{minipage}%

\subsubsection{Visser predictor pseudo-code}

\noindent\begin{minipage}[t]{0.95\textwidth}
  \begin{algorithm}[H]
\begin{lstlisting}[language=Python, label={alg:visser}, caption=Pseudo-code for Visser predictor.]
def visser(m, n, eta):

  p = 1/(2.0 * eta) * jax.numpy.eye(m.shape[0])
  for _ in range(n):
    p += eta * (m - jax.numpy.matmul(p, p))

  return p

z_theta = online_out['projection_view1']
sigma = jax.numpy.matmul(jax.numpy.transpose(z_theta), z_theta)

pred = visser(sigma, n_iters, epsilon)

\end{lstlisting}
\end{algorithm}
\end{minipage}%

\subsubsection{Newton-Schulz predictor pseudo-code}

\noindent\begin{minipage}[t]{0.95\textwidth}
  \begin{algorithm}[H]
\begin{lstlisting}[language=Python, label={alg:ns}, caption=Pseudo-code for Newton-Schulz predictor.]
def newtonschulz(m, n):

  a = m
  fronorm_a = jax.numpy.linalg.norm(a, 'fro')
  a /= fronorm_a

  b = jax.numpy.eye(m.shape[0])
  c = 3.0 * b
  for _ in range(n):
    ba = jax.numpy.matmul(b, a)
    a = 0.5 * jax.numpy.matmul(a, c - ba)
    b = 0.5 * jax.numpy.matmul(c - ba, b)

  a *= jax.numpy.sqrt(fronorm_a)
  b /= jax.numpy.sqrt(fronorm_a)
  
  return a, b
  
z_theta = online_out['projection_view1']
sigma = jax.numpy.matmul(jax.numpy.transpose(z_theta), z_theta)

pred = newtonschulz(sigma, n_iter)[0]
  
\end{lstlisting}
\end{algorithm}
\end{minipage}%

\subsubsection{Stiefel predictor pseudo-code}

\noindent\begin{minipage}[t]{0.95\textwidth}
  \begin{algorithm}[H]
\begin{lstlisting}[language=Python, label={alg:stiefel}, caption=Pseudo-code for Stiefel predictor.]
def stiefel(m, n):
  x = m.transpose()
  xxt = jax.numpy.matmul(x, m)
  xxt /= m.shape[0]

  x_inv_sqrt = newtonschulz(xxt, n)[1]
  x_inv_sqrt /= jax.numpy.linalg.norm(x_inv_sqrt, ord='fro')

  p = jax.numpy.matmul(x_inv_sqrt, x)

  return p 
  
z_theta = online_out['projection_view1']
sigma = jax.numpy.matmul(jax.numpy.transpose(z_theta), z_theta)

pred = stiefel(sigma, n_iter)
    
\end{lstlisting}
\end{algorithm}
\end{minipage}%

\subsubsection{Linear Regression Predictor pseudo-code}

\noindent\begin{minipage}[t]{0.95\textwidth}
  \begin{algorithm}[H]
\begin{lstlisting}[language=Python, label={alg:LRP}, caption=Pseudo-code for LRP predictor.]
def lrp(zt, zx, safe_eps=1e-12):
    normfactor = 0.5 * jax.numpy.linalg.norm(zt, ord='fro')
    normfactor += 0.5 * jax.numpy.linalg.norm(zx, ord='fro') + safe_eps
    
    p = jax.numpy.matmul(jax.numpy.linalg.pinv(zt/normfactor), zx/normfactor)
    return p

z_theta = online_out['projection_view1']
z_xi = target_out['projection_view2']
pred = lrp(z_theta, z_xi)



\end{lstlisting}
\end{algorithm}
\end{minipage}%

\newpage
\section{Additional quantitative results}

\subsection{Top-$5$ evaluation of predictors}

\begin{table}[ht]
\centering
\scalebox{0.85}{
\begin{tabular}{l|cccccccc}
Predictor   & $0.0$ & $0.15$ & $0.3$ & $0.45$ & $0.6$ & $0.75$ & $0.9$  \\ \midrule
$\text{DirectCopy}^{\dagger}$            & $43.6$ & $87.4$ & $88.1$ & $\underline{88.3}$ & $88.3$ & $88.2$ & $88.2$ \\ \midrule
NE                   & $84.4$ & $\underline{88.9}$ & $88.7$ & $88.6$ & $88.4$ & $88.2$ & $88.0$   \\
Visser                & $\underline{88.8}$ & $88.7$ & $88.8$ & $88.8$ & $88.8$ & $88.8$ & $88.8$ \\
NS  $,n=9$                & $87.6$ & $87.7$ & $87.8$ & $88.1$ & $88.4$ & $88.4$ & $\underline{88.5}$ \\
$\text{NS}^{2}$ $,n=7$ & $87.9$ & $88.8$ & $88.9$ & $89.0$   & $88.9$ & $89.0$   & $\underline{89.0}$   \\
Stiefel $,n=9$              & $87.8$      & $88.4$     & $89.0$      & $89.0$      & $89.0$     & $\underline{\mathbf{89.1}}$      & $89.0$ \\ \bottomrule

\end{tabular}
}
\vspace{-1.0mm}
\caption{\label{hundredepochstop5}Top-$5$ accuracy on ImageNet-$1$k (linear protocol) of our predictors, at $100$ epochs, as ridging $\alpha$ varies. \emph{DirectCopy} \citep{Wang2021TowardsDR} is our own replication. Average over $3$ random seeds. Standard trainable linear predictor \byol achieves $88.5\%$.}
\end{table}

\begin{table}[ht]
\centering
\scalebox{0.85}{
\begin{tabular}{l|cccccccc}
Predictor  & $0.0$ & $0.15$ & $0.3$ & $0.45$ & $0.6$ & $0.75$ & $0.9$  \\ \midrule
$\text{DirectCopy}^{\dagger}$          & $74.4$ & $90.5$ & $90.6$ & $90.6$ & $90.6$ & $\underline{90.7}$ & $90.6$ \\ \midrule
NE                   & $89.2$   & $\underline{90.8}$ & $90.6$ & $90.6$ & $90.4$ & $90.3$ & $90.2$ \\
Visser                & $\underline{90.6}$ & $90.6$ & $90.6$ & $90.6$ & $90.5$ & $90.6$ & $90.5$ \\
NS  $,n=9$                & $90.2$ & $90.8$ & $\underline{90.9}$ & $90.9$ & $90.8$ & $90.9$ & $90.9$ \\
$\text{NS}^{2}$ $,n=7$ & $90.3$ & $90.9$ & $\underline{91.0}$   & $91.0$   & $90.9$ & $90.9$ & $90.9$ \\
Stiefel $,n=9$               & $90.5$      & $\underline{\mathbf{91.1}}$      & $90.9$     & $90.9$      & $91.0$     & $90.9$      & $90.9$      \\ \bottomrule
\end{tabular}
}
\vspace{-1.0mm}
\caption{\label{threehundredepochstop5}Top-$5$ accuracy on ImageNet-$1$k (linear protocol) of our predictors, at $300$ epochs, as ridging $\alpha$ varies. \emph{DirectCopy} \citep{Wang2021TowardsDR} is our own replication. Average over $3$ random seeds. Standard trainable linear predictor \byol achieves $91.0\%$.}
\end{table}

\subsection{Ablation on the number of iterations in approximate methods}\label{appx:ablation}

\textbf{Ablation on the number of iterations.} It is instructive to vary the number of iterations for our approximate square-root methods. Our baseline Newton-Schulz predictor uses $9$ iterations. In Table \ref{nsiterations}, we compute the performance as a function of variable number of iterations $n$. We see it remains robust both at $100$ and $300$ epochs as long as $n \geq 5$, and that variations in performance remain minimal, in line with intuition for a second-order optimization method. Good results are obtained with as few as $3$ iterations. In Table \ref{visseriterations}, we perform the same analysis with the number of necessary Visser iterations. We find that at $100$ epochs, which is the case we're most interested in, at least $50$ iterations are required for good performance, and more iterations doesn't hurt performance. No such clear pattern regarding performance emerges at $300$ epochs.

\begin{table}[!hb]
\centering
\small
\begin{tabular}{l|cccc}
Visser performance & $12$ & $25$ & $50$ & $100$ \\ \midrule
$100$ epochs    &$67.2$& $68.2$  &  $\mathbf{68.4}$ &$68.4$        \\
$300$ epochs    & $\mathbf{71.8}$& $70.6$ &$71.6$& $69.2$ \\ \bottomrule
\end{tabular}
\caption{\label{visseriterations}Top-$1$ accuracy on ImageNet-$1$k (linear protocol) of the Visser predictor, as a function of its number of iterations, for best ridging parameter $\alpha=0.0$. Average over $3$ random seeds.}
\end{table}

\begin{table}[!hb]
\centering
\small
\begin{tabular}{l|cccc}
NS performance & $3$ & $5$ & $7$ & $9$\\ \midrule
$100$ epochs    &$68.0$& $\mathbf{68.1}$  &  $\mathbf{68.1}$ &$68.0$        \\
$300$ epochs    & $71.8$& $\mathbf{72.1}$ &$72.0$& $71.9$   \\ \bottomrule
\end{tabular}
\caption{\label{nsiterations}Top-$1$ accuracy on ImageNet-$1$k (linear protocol) of the NS predictor, as a function of its number of iterations, for best ridging parameter $\alpha=0.9$. Average over $3$ random seeds.}
\end{table}

\begin{table}[!ht]
\centering
\small
\begin{tabular}{l|cccc}
Stiefel performance & $3$ & $5$ & $7$ & $9$\\ \midrule
$100$ epochs    &$66.8$& $68.0$  &  $68.6$ & $\mathbf{68.8}$        \\
$300$ epochs    & $71.8$& $71.9$ &$71.9$& $\mathbf{72.1}$   \\ \bottomrule
\end{tabular}
\caption{\label{stiefeliterations}Top-$1$ accuracy on ImageNet-$1$k (linear protocol) of the Stiefel predictor, as a function of its number of iterations, for best ridging parameter $\alpha=0.3$. Average over $3$ random seeds.}
\end{table}

When all three spectrums of epochs, number of iterations as we've just seen Table \ref{stiefeliterations}, and ridging parameter $\alpha$ are considered, the Stiefel predictor performance is best across the board.

\section{Additional graphs}

Figures \ref{fig:DSV_four} and \ref{fig:distrib_four} extend Figure \ref{fig:distrib_main}. For four different predictors (\emph{DirectCopy}, and our NE, NS and Stiefel predictors), they depict the evolution throughout training of the empirical distribution of singular values of the online latents' covariance, as those are dynamically renormalized between $0$ and $1$ throughout training. These distributions are averaged across $3$ ImageNet trials at $100$ training epochs.  Starting from a unimodal distribution with a very concentrated peak near $0$, these distributions become bimodal with local maxima near $0$ and $1$ throughout training. This is consistent with the expected behaviour of a matrix getting closer and closer to idempotence ($P^2 = P$).

We observe first Figure \ref{fig:DSV_four} that a narrow peak of eigenvalues close to $1$, indicative of a predictor almost perfectly an orthogonal projection, seems to loosely correlate with the final performance reached by that predictor. Second, we see Figure \ref{fig:distrib_four} that there are two distinct temporal profiles for the predictors we depict, and the distributions of singular values they induce. The top $2$ (NE and \emph{DirectCopy}) latch onto a few key singular directions early in training, before making steady progress afterwards. By contrast, the two bottom predictors depicted in that figure, NS and Stiefel, both based on the Newton-Schulz iteration, display a clear pattern of wider selection of singular directions (meaning that the span of these predictors is slower to grow, until midway in training where the behaviour changes abruptly). This behaviour is also associated with slightly better terminal accuracy under linear evaluation.

\begin{figure}
     \centering
     \begin{subfigure}[ht]{0.475\textwidth}
         \centering
         \includegraphics[width=\textwidth]{plots/DSV-DirectCopy.png}
         \caption{\emph{DirectCopy} predictor}
         \label{fig:DSV_DC}
     \end{subfigure}
     \begin{subfigure}[ht]{0.475\textwidth}
         \centering
         \includegraphics[width=\textwidth]{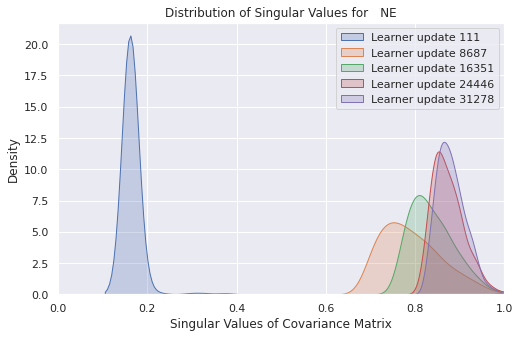}
         \caption{NE predictor}
         \label{fig:DSV_NE}
     \end{subfigure}
     \vfill
     \begin{subfigure}[hb]{0.45\textwidth}
         \centering
         \includegraphics[width=\textwidth]{plots/DSV-NS.png}
         \caption{Newton-Schulz predictor}
         \label{fig:DSV_NS}
     \end{subfigure}
     \begin{subfigure}[hb]{0.45\textwidth}
         \centering
         \includegraphics[width=\textwidth]{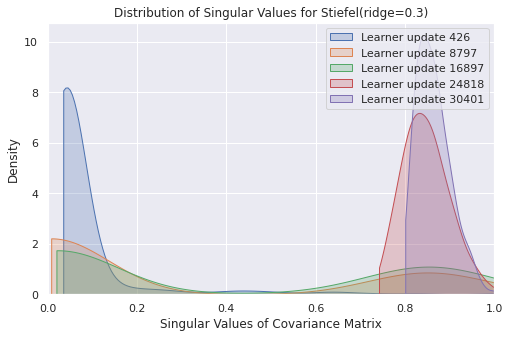}
         \caption{Stiefel predictor}
         \label{fig:DSV_stiefel}
     \end{subfigure}
        \caption{Empirical distribution of the singular values of the covariance of features, for various predictors. Plot throughout $100$ epochs of training, on ImageNet-$1k$, averaged over $3$ random seeds, extending Figure \ref{fig:distrib_main}.}
        \label{fig:DSV_four}
\end{figure}

\begin{figure}
     \centering
     \begin{subfigure}[ht]{0.475\textwidth}
         \centering
         \includegraphics[width=\textwidth]{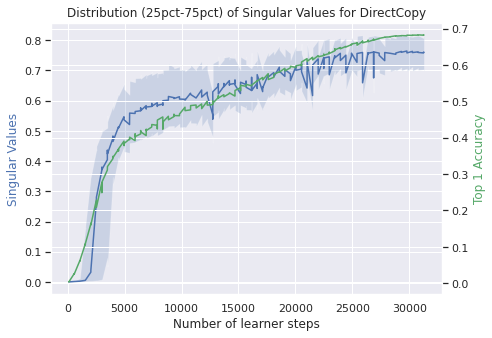}
         \caption{\emph{DirectCopy} predictor}
         \label{fig:distrib_DC}
     \end{subfigure}
     \begin{subfigure}[ht]{0.475\textwidth}
         \centering
         \includegraphics[width=\textwidth]{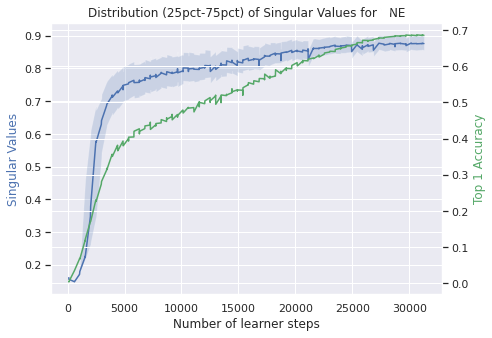}
         \caption{NE predictor}
         \label{fig:distrib_NE}
     \end{subfigure}
     \vfill
     \begin{subfigure}[hb]{0.45\textwidth}
         \centering
         \includegraphics[width=\textwidth]{plots/Distribution-NS.png}
         \caption{Newton-Schulz predictor}
         \label{fig:distrib_NS}
     \end{subfigure}
     \begin{subfigure}[hb]{0.45\textwidth}
         \centering
         \includegraphics[width=\textwidth]{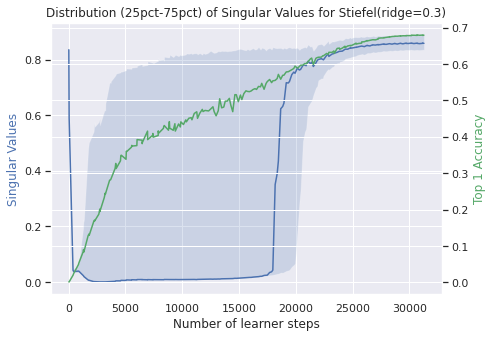}
         \caption{Stiefel predictor}
         \label{fig:distrib_stiefel}
     \end{subfigure}
        \caption{\emph{Blue} (left legend): Median, $25$th and $75$th percentile of empirical distribution of the singular values of the empirical covariance of latents, for four different predictors. \emph{Green} (right legend): Evolution of their Top-$1$ $\%$ linear accuracy. Plot throughout $100$ epochs of training, on ImageNet-$1k$, averaged over $3$ random seeds, extending Figure \ref{fig:distrib_main}.}
        \label{fig:distrib_four}
\end{figure}

\section{Compute costs}

We run our experiments on tranches of $128$ Cloud TPUv3 cores. With this number of devices, each single run usually lasts a few hours, depending on its number of epochs; a typical result with $3$ random seeds over $300$ epochs each requires less than a day of compute in this setting. As we use the LARS optimizer~\citep{You2017LargeBT}, this configuration is robust to lowering the total number of devices without much retuning of the base learning rate, although this comes at a proportional cost in terms of elapsed compute time.

\section{Acknowledgements}

We are extremely appreciative of the collaborative research environment within DeepMind. Many thanks to Matko Bošnjak for his feedback on an earlier draft of this paper, as well as Andr\'as György, Andrew Lampinen, Bernardo \'Avila Pires, Mark Rowland, Mohammad Gheshlaghi Azar, and Zhaohan Daniel Guo for helpful and interesting related conversations.

\end{appendix}

\end{document}